\documentclass[11pt, a4paper]{amsart}

\usepackage[dvipsnames]{xcolor}
\usepackage[pagebackref,backref=true,colorlinks=true, linkcolor=Sepia, citecolor=Sepia]{hyperref}
\usepackage{amsmath,amsthm,amssymb,fancyhdr,graphicx,bbm,cancel,mathrsfs,todonotes,xypic,pinlabel}
\usepackage{tikz}
\usepackage{extarrows,tipa}
\usetikzlibrary{matrix}
\usepackage[all]{xy}
\usepackage{mathtools}
\usepackage{bm}
\usepackage{verbatim}
\usepackage{microtype}
\usepackage{subcaption}

\DeclareFontFamily{T1}{cbgreek}{}
\DeclareFontShape{T1}{cbgreek}{m}{n}{<-6>  grmn0500 <6-7> grmn0600 <7-8> grmn0700 <8-9> grmn0800 <9-10> grmn0900 <10-12> grmn1000 <12-17> grmn1200 <17-> grmn1728}{}
\DeclareSymbolFont{quadratics}{T1}{cbgreek}{m}{n}
\DeclareMathSymbol{\qoppa}{\mathord}{quadratics}{19}
\DeclareMathSymbol{\Qoppa}{\mathord}{quadratics}{21}

\usepackage{amsthm}

\newtheorem{lemma}{Lemma}
\newtheorem{proposition}{Proposition}
\newtheorem{corollary}{Corollary}
\newtheorem{definition}{Definition}

\newtheorem{remark}{Remark}
\newtheorem{assumption}{Assumption}

\def\b{\ensuremath\boldsymbol}

\usepackage{mathtools} 

\usepackage{enumerate,pdfpages,stmaryrd} 
\usepackage[margin=1.2in,footskip=.5in,marginparwidth=1in]{geometry}

\usepackage{tikz}
\usepackage{dsfont}
\usetikzlibrary{matrix}

\newcommand{\beq}{\begin{equation*}}
\newcommand{\eeq}{\end{equation*}}

\begin{document}

\title{On Manifold Hypothesis: Hypersurface Submanifold Embedding Using Osculating Hyperspheres}

\author{Benyamin Ghojogh, Fakhri Karray, Mark Crowley}

\maketitle
	
\begin{abstract}
Consider a set of $n$ data points in the Euclidean space $\mathbb{R}^d$. This set is called dataset in machine learning and data science. Manifold hypothesis states that the dataset lies on a low-dimensional submanifold with high probability. All dimensionality reduction and manifold learning methods have the assumption of manifold hypothesis. In this paper, we show that the dataset lies on an embedded hypersurface submanifold which is locally $(d-1)$-dimensional. Hence, we show that the manifold hypothesis holds at least for the embedding dimensionality $d-1$. Using an induction in a pyramid structure, we also extend the embedding dimensionality to lower embedding dimensionalities to show the validity of manifold hypothesis for embedding dimensionalities $\{1, 2, \dots, d-1\}$. For embedding the hypersurface, we first construct the $d$ nearest neighbors graph for data. For every point, we fit an osculating hypersphere $S^{d-1}$ using its neighbors where this hypersphere is osculating to a hypothetical hypersurface. Then, using surgery theory, we apply surgery on the osculating hyperspheres to obtain $n$ hyper-caps. We connect the hyper-caps to one another using partial hyper-cylinders. By connecting all parts, the embedded hypersurface is obtained as the disjoint union of these elements. We discuss the geometrical characteristics of the embedded hypersurface, such as having boundary, its topology, smoothness, boundedness, orientability, compactness, and injectivity. Some discussion are also provided for the linearity and structure of data. This paper is the intersection of several fields of science including machine learning, differential geometry, and algebraic topology. 
\end{abstract}

\section{Introduction}









Suppose we have a set of data points in a multi-dimensional Euclidean space. 
This paper is motivated by the following question. Do the points of this dataset lie on a submanifold? This question is answered by the manifold hypothesis \cite{fefferman2016testing}. According to the manifold hypothesis, the data points most often lie on a submanifold with high probability. This is because the data points usually represent some natural signal such as image. When the data acquisition process is natural, this hypothesis makes more sense because the data will have some structure. 

The manifold hypothesis is the assumption of all manifold learning and dimensionality reduction methods in the fields of machine learning and data science. They all assume that data points lie on some submanifold with an intrinsic dimensionality. These methods try to find this submanifold with different approaches. 
The linear manifold learning methods, such as principal component analysis \cite{ghojogh2019unsupervised}, classical multidimensional scaling \cite{cox2008multidimensional}, and Fisher discriminant analysis \cite{fisher1936use}, assume that this submanifold is linear. However, the nonlinear manifold learning methods, such as Isomap \cite{tenenbaum2000global}, locally linear embedding \cite{roweis2000nonlinear}, and Laplacian eigenmap \cite{belkin2003laplacian}, do not have such assumption. 

A comprehensive algorithm has been proposed for testing whether manifold hypothesis holds for some dataset, with high probability \cite{fefferman2016testing}. In its formulation, it uses the Johnson-Lindenstrauss lemma \cite{johnson1984extensions}, which is also used in random projection \cite{ghojogh2021johnson}. 
Here, in this paper, we discuss the manifold hypothesis by hypersurface submanifold embedding using osculating hyperspheres. 
Section \ref{section_definitions} reviews and provides some definitions. The proposed discussion on hypersurface embedding is explained in Section \ref{section_hypersurface_embedding}. We provide a numerical example in Section \ref{section_example}. The geometrical characteristics of the embedded hypersurface are discussed in Section \ref{section_geometrical_characteristics}.
In Section \ref{section_induction}, we use induction in a pyramid structure for extension to lower embedding dimensionalities. 
Some discussions on the linearity and structure of data are provided in Section \ref{section_discussions}. 
Finally, Section \ref{section_conclusion} concludes the paper. 

\section{Definitions and Background}\label{section_definitions}

\subsection{Manifold and Submanifold}

\begin{definition}[Topology and topological space \cite{lee2010introduction,kelley2017general}]
Let $\mathcal{X}$ be a set. A topology on $\mathcal{X}$ is a collection $\mathcal{T}$ of subsets $\mathcal{X}$, called open sets, satisfying:
\begin{itemize}
\item $\varnothing, \mathcal{X} \in \mathcal{T}$ 
\item If $U_1, \dots, U_k \in \mathcal{T}$, then $\bigcap_{j=1}^k U_j \in \mathcal{T}$. In other words, finite intersections of open sets are open. 
\item If $U_\alpha \in \mathcal{T}, \forall \alpha \in A$ (where $A$ is the index set of topology), then $\bigcup_{\alpha \in A} U_\alpha \in \mathcal{T}$. In other words, arbitrary unions of open sets are open. 
\end{itemize}
The pair $(\mathcal{X}, \mathcal{T})$ is called a topological space associated with the topology $\mathcal{T}$. 
\end{definition}

\begin{definition}[Hausdorff space \cite{lee2010introduction,kelley2017general}]
A topological space $(\mathcal{X}, \mathcal{T})$ is Hausdorff if and only if for $x_1, x_2 \in X$, $x_1 \neq x_2$, we have:
\begin{align}
\exists\, \text{open sets } U,V \text{ such that } x_1 \in U,\, x_2 \in V,\, U \cap V = \varnothing.
\end{align}
In other words, the points of a Hausdorff topological space are separable and distinguishable. 
\end{definition}

\begin{definition}[Topological manifold \cite{lee2010introduction}]
A topological space $(\mathcal{X}, \mathcal{T})$ is a topological manifold of dimension $d$, for $d \in \mathbb{Z}_{\geq 0}$, also called a topological $d$-manifold, if all the following conditions hold:
\begin{itemize}
\item $(\mathcal{X}, \mathcal{T})$ is Hausdorff. 
\item $(\mathcal{X}, \mathcal{T})$ has a countable basis. 
\item $(\mathcal{X}, \mathcal{T})$ is locally homeomorphic to $d$-dimensional Euclidean space, $\mathbb{R}^d$. 
\end{itemize}
\end{definition}

\begin{definition}[Chart \cite{lee2010introduction}]
Consider a topological manifold $\mathcal{M} := (\mathcal{X}, \mathcal{T})$. It is locally homeomorphic to $\mathbb{R}^d$, meaning that for all $x \in X$, there exists an open set $U$ containing $x$ and a homeomorphism $\phi: U \rightarrow \phi(U)$ where $\phi(U)$ is an open subset of $\mathbb{R}^d$. Such mapping is denoted by $\phi: U \overset{\cong}{\longrightarrow} \phi(U)$ and the tuple $(U, \phi)$ is called a coordinate chart, or a chart in short, for $\mathcal{M}$. 
\end{definition}

\begin{definition}[Smooth atlas \cite{lee2013smooth}]
A smooth atlas $\mathcal{A}$ for a topological $d$-manifold $\mathcal{M}$ is a collection of charts $(U_\alpha, \phi_\alpha)$ for $\mathcal{M}$ such that:
\begin{itemize}
\item They cover $\mathcal{M}$, i.e., $\bigcup_{\alpha \in A} U_\alpha = \mathcal{M}$.
\item Any two charts in this collection are smoothly compatible (n.b. two charts $(U, \phi)$ and $(V, \psi)$ are smoothly compatible if the mapping $\psi \circ \phi^{-1}$ is a diffeomorphism).
\end{itemize}
\end{definition}

\begin{definition}[Maximal atlas \cite{lee2013smooth}]
A smooth atlas $\mathcal{A}$ for a topological $d$-manifold $\mathcal{M}$ is maximal if it is not contained in any other smooth atlas for $\mathcal{M}$. 
\end{definition}

\begin{definition}[Smooth manifold \cite{lee2013smooth}]
A smooth manifold $\mathcal{M}$ of dimension $d$, also called a smooth $d$-manifold, is a topological $d$-manifold together with a choice of maximal smooth atlas $\mathcal{A}$ on $\mathcal{M}$. 
\end{definition}

\begin{definition}[Compact topological space \cite{lee2010introduction,lee2013smooth}]\label{definition_compact_topological_space}
A topological space $(\mathcal{X}, \mathcal{T})$ is compact if $\mathcal{X}$ is the union of a collection of open sets and there exists a finite sub-collection whose union is $\mathcal{X}$. In other words, in a compact topological space, $\mathcal{X}$ has a finite sub-cover. 
A topological manifold is compact if it is a compact topological space. Compact manifolds are usually manifolds without boundary. 
\end{definition}

\begin{definition}[Subspace topology \cite{lee2010introduction,lee2013smooth}]
Consider a topological space $(\mathcal{X}, \mathcal{T})$ and a subset $\mathcal{S} \subseteq \mathcal{X}$. The subspace topology on $\mathcal{S}$ is defined as $\mathcal{T}_{\mathcal{S}} := \{\mathcal{S} \cap U \,|\, U \in \mathcal{T}\}$. There is an inclusion map $\mathcal{S} \hookrightarrow \mathcal{X}$ for a subspace topology. 
\end{definition}

\begin{definition}[Submersion, immersion, and embedding {\cite[Chapter 4]{lee2013smooth}}]
Let $\mathcal{M}$ and $\mathcal{N}$ be two smooth manifolds. 
\begin{itemize}
\item A smooth map $F: \mathcal{M} \rightarrow \mathcal{N}$ is a smooth submersion if its differential is surjective, i.e., $\text{rank}(F) = \text{dim}(\mathcal{N})$, where $\text{dim}(.)$ denotes the local dimensionality of manifold. In submersion, we have $\text{dim}(\mathcal{M}) \geq \text{dim}(\mathcal{N})$.
\item A smooth map $F: \mathcal{M} \rightarrow \mathcal{N}$ is a smooth immersion if its differential is injective, i.e., $\text{rank}(F) = \text{dim}(\mathcal{M})$. In immersion, we have $\text{dim}(\mathcal{M}) \leq \text{dim}(\mathcal{N})$.
\item A smooth map $F: \mathcal{M} \rightarrow \mathcal{N}$ is a topological embedding if it is a homeomorphism to its image $F(\mathcal{M}) \subseteq \mathcal{N}$ in the subspace topology. 
\item A smooth map $F: \mathcal{M} \rightarrow \mathcal{N}$ is a smooth embedding if it is both a smooth immersion and a topological embedding. An example smooth embedding is the inclusion map $\mathcal{S} \hookrightarrow \mathcal{M}$ where $\mathcal{S} \subseteq \mathcal{M}$.
\end{itemize}
\end{definition}

\begin{definition}[Embedded submanifold {\cite[Chapter 5]{lee2013smooth}}]
Let $\mathcal{M}$ be a smooth manifold. An embedded submanifold of $\mathcal{M}$ is a subset $\mathcal{S} \subseteq \mathcal{M}$ which is itself a manifold endowed with a smooth structure where the inclusion map $\mathcal{S} \hookrightarrow \mathcal{M}$ is a smooth embedding.
The quantity $\text{dim}(\mathcal{M}) - \text{dim}(\mathcal{S})$ is called the codimension of $\mathcal{S}$ in $\mathcal{M}$.
\end{definition}

For more information on topological and smooth manifolds, the reader can refer to \cite{lee2010introduction} and \cite{lee2013smooth}, respectively. 

\subsection{Hypersurface and Hypersphere}

\begin{definition}[Embedded hypersurface {\cite[Chapter 5]{lee2013smooth}}]\label{definition_embedded_hypersurface}
An embedded hypersurface is an embedded submanifold $\mathcal{S} \subseteq \mathcal{M}$ with the inclusion map $\mathcal{S} \hookrightarrow \mathcal{M}$, whose codimension is one. 
For example, if $\mathcal{M} = \mathbb{R}^d$, an embedded hypersurface is locally homeomorphic to $\mathbb{R}^{d-1}$.
Note that a hypersurface is topologically homeomorphic to a hyperplane; although, it is not necessarily linear. 
\end{definition}

\begin{definition}[$d$-sphere]\label{definition_d_sphere}
The $d$-sphere, denoted by $S^d$, is a hypersphere which is locally $d$-dimensional and is embedded in $\mathbb{R}^{d+1}$, i.e.:
\begin{align}
S^d := \{\b{x} \in \mathbb{R}^{d+1} \,|\, \Vert \b{x} \Vert = 1\},
\end{align}
where $\Vert \cdot \Vert$ denotes a norm in the Euclidean space. 
\end{definition}

\begin{lemma}\label{lemma_number_of_points_for_d_sphere}
It is clear that $d+2$ points lie on a unique $S^d$; hence, $d+2$ points are needed to fit a unique $S^d$ to them. 
This is because $S^d$ can be seen as an embedded submanifold in $\mathbb{R}^{d+1}$; hence, it requires $(d+1+1)$ points to be defined uniquely. 
\end{lemma}

\subsection{Osculating Hypersphere}


\hfill\break
\hfill\break
Osculating circle was proposed by Leibniz and Newton in the 17-th century. It is defined below. 
\begin{definition}[Osculating circle {\cite[Proposition V, Problem I]{newton1687principia}}]\label{definition_osculating_circle}
Consider a curve which is locally one-dimensional. At every point $\b{x}$ of this curve, we have a tangent circle which is tight to the curve in the sense that for two points $\b{x}_1$ and $\b{x}_2$ before and after $\b{x}$ on the curve, this circle passes through the three points $\b{x}$, $\b{x}_1$, and $\b{x}_2$ while $\b{x}_1$ and $\b{x}_2$ tend to $\b{x}$ on the curve. This tangent circle is called the osculating circle.
If the radius of the osculating circle at $\b{x}$ is $r$, the curvature of curve at $\b{x}$ is defined to be $1/r$. 
\end{definition}


An important characteristic of the osculating circles is as follows. 
\begin{lemma}[The Tait-Kneser theorem \cite{ghys2013osculating}]
Consider a smooth curve which is locally one-dimensional. If it has monotonic curvature, then the osculating circles of the curve are disjoint and nested within each other.
\end{lemma}


The osculating circle is closely related to the involute and evolute of curve \cite{mccleary2013geometry}, defined below. \begin{definition}[Involute and evolute of curve \cite{huygens1673horologium}]
The involute of a curve is the locus of points at the tight string to the curve which is unwrapped from the curve. 
The locus of centers of the osculating circles, while moving along the curve, is the evolute of curve. The evolute of an involute of a curve is the curve itself. 
\end{definition}

In this paper, we are not restricted to 1-manifolds; hence, we generalize the definition of the osculating circles to osculating hyperspheres, defined below.  
\begin{definition}[Osculating hypersphere]\label{definition_osculating_hypersphere}
Consider a $d$-hypersurface which is $d$-dimensional locally and is embedded in $\mathbb{R}^{d+1}$. At every point $\b{x}$ of this hypersurface, we have a tangent $d$-sphere which is tight to the manifold in the sense that for $d+1$ points $\{\b{x}_1, \dots, \b{x}_{d+1}\}$ around $\b{x}$ on the manifold, this $d$-sphere passes through the points $\b{x}$ and $\{\b{x}_1, \dots, \b{x}_{d+1}\}$ while all the points $\{\b{x}_1, \dots, \b{x}_{d+1}\}$ tend to $\b{x}$ on the manifold. 
This tangent $d$-sphere is called the osculating hypersphere.
If the radius of the osculating hypersphere at $\b{x}$ is $r$, the curvature of hypersurface at $\b{x}$ is defined to be $1/r$. 
\end{definition}

\subsection{Manifold Hypothesis}


\begin{definition}[Dataset]
Consider a set of data points lying in the Euclidean space $\mathbb{R}^d$. This set is denoted by $\mathcal{D} := \{\b{x}_i\}_{i=1}^n$ where $\b{x}_i = [x_i^1, \dots, x_i^d]^\top$ is the $i$-th data point and $x_i^j$ denotes the $j$-th dimension (or feature) of $\b{x}_i$. This set is called dataset in machine learning and data science. 
\end{definition}

\begin{assumption}\label{assumption_dataset_in_ball}
We assume that the dataset is bounded, meaning that it lies in a ball $\mathcal{B}$ in $\mathbb{R}^d$: 
\begin{align}
& \mathcal{B}^d_r := \{\b{x} \in \mathbb{R}^{d} \,|\, \Vert \b{x} \Vert \leq r\}, \\
& \mathcal{D} \subset \mathcal{B}^d_r, 
\end{align}
where $r \in (0, \infty)$ is the radius of ball which may be very large but not infinite.  
\end{assumption}

\begin{lemma}[Whitney embedding theorem \cite{whitney1936differentiable,whitney1944self}]
Every $d$-dimensional differentiable manifold can be embedded in $\mathbb{R}^{2d+1}$ \cite{whitney1936differentiable}. In some cases, it can be embedded in $\mathbb{R}^{2d}$ \cite{whitney1944self}.
\end{lemma}

\begin{definition}[Manifold hypothesis \cite{fefferman2016testing}]
According to the manifold hypothesis, data points of a dataset lie on a submanifold with lower dimensionality. In other words, the dataset in $\mathbb{R}^d$ lies on an embedded submanifold with local dimensionality less than $d$. 
\end{definition}

\section{Hypersurface Embedding}\label{section_hypersurface_embedding}

Consider a dataset in the Euclidean space $\mathbb{R}^d$. 
Here, we explain our proposed method for hypersurface embedding on which the dataset lies. By this hypersurface, we show that manifold hypothesis holds at least with dimensionality $d-1$ for a dataset in $\mathbb{R}^d$. 

\subsection{Main Idea}

\hfill\break
\hfill\break
The main idea of hypersurface embedding is as follows. 
We use the idea of osculating hyperspheres. However, a problem arises here. We do not have a hypersurface yet to define the osculating hyperspheres on (see Definition \ref{definition_osculating_hypersphere}). In other words, a chicken and egg problem exists here. For our goal in finding a hypersurface, we need osculating hyperspheres while osculating hyperspheres require a hypersurface to be defined on.
For resolving this issue, we go the other way around by defining a hypersurface based on several existing osculating hyperspheres. In other words, we find some osculating hyperspheres for the dataset and then fit a hypersurface to those hyperspheres in a way that the hyperspheres would be osculating to the hypersurface. This is shown, by an example in $\mathbb{R}^2$, in Fig. \ref{figure_osculating_circles}.

\begin{figure}[!h]
\centering
\includegraphics[width=5in]{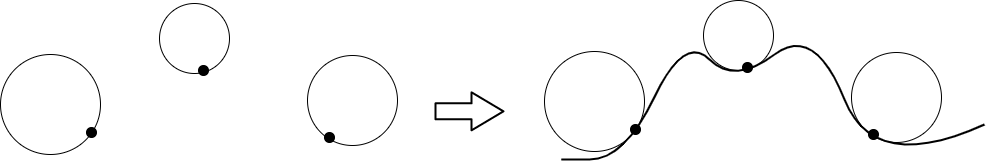}
\caption{Fitting a hypothetical hypersurface to several osculating hyperspheres in $\mathbb{R}^2$.}
\label{figure_osculating_circles}
\end{figure}

We assume all points of the dataset lie on a hypothetical hypersurface. 
For every point in the dataset, we fit an osculating hypersphere using its neighbors. Therefore, we use a $k$-Nearest Neighbors ($k$NN) graph for the dataset, where $k=d$ in our method. 
For every point $\b{x}$, we fit an osculating hypersphere to it and its $k$ neighbors while each of its neighbors tend to the point $\b{x}$ on a straight line connecting the neighboring point and $\b{x}$. When all the $k$ neighbors tend to the point $\b{x}$, the hypersphere converges to the osculating hypersphere on a hypothetical hypersurface at the point $\b{x}$. 

We fit an osculating hypersphere to every point in the dataset to have $n$ number of osculating hyperspheres. Then, we fit a hypersurface to the existing osculating hyperspheres to obtain the hypothetical hypersurface. The obtained hypersurface is the desired hypersurface which verifies the manifold hypothesis in at least $\mathbb{R}^{d-1}$. 
For fitting the hypersurface to osculating hyperspheres, we do surgery on every hypersphere using surgery theory \cite{browder2012surgery}. For every osculating hypersphere, we keep only a hyper-cap containing every point and ignore the rest of hypersphere. Then, we stitch the hyper-caps to each other using hyper-cylinders between the caps. The resulted manifold, whose parts are the hyper-caps and the hyper-cylinders, is the desired smooth hypersurface which is locally $(d-1)$-dimensional. 


\subsection{Fitting Osculating Hyperspheres to Points}

\hfill\break
\hfill\break
The dataset $\mathcal{D} :=\{\b{x}_i\}_{i=1}^n$ is in $\mathbb{R}^d$ so, according to Definition \ref{definition_embedded_hypersurface}, the desired hypersphere is locally $(d-1)$-dimensional. According to Definition \ref{definition_d_sphere}, the osculating hyperspheres for this desired hypersurface is $S^{d-1}$ which is locally $(d-1)$-dimensional, embedded in $\mathbb{R}^d$. According to Lemma \ref{lemma_number_of_points_for_d_sphere}, we require $d+1$ points for uniquely fitting such an osculating hypersphere. 

For every point $\b{x}$, we want to fit an osculating hypersphere using the point itself and its $k$NN. As we need $d+1$ points for fitting a unique hypersphere, $k$ should be equal to $d$:
\begin{align}
k = d.
\end{align}
This also means that the following assumption should hold. 
\begin{assumption}
A required assumption for our method is to have:
\begin{align}
n > d,
\end{align}
i.e., the sample size of dataset is larger than the dimensionality of dataset. 
\end{assumption}
Let the $d$ neighbors of the $i$-th data point $\b{x}_i \in \mathcal{D}$ be denoted by $\{\b{x}_{i,j}\}_{j=1}^d$ where $\b{x}_{i,j} \in \mathcal{D}, \forall i,j$. According to Lemma \ref{lemma_number_of_points_for_d_sphere}, every point $\b{x}_i$ with its neighbors $\{\b{x}_{i,j}\}_{j=1}^d$ can define a unique $(d-1)$-sphere, embedded in $\mathbb{R}^d$, passing through them. Let this $(d-1)$-sphere be denoted by $S_i$ corresponding to $\b{x}_i$. Through the point $\b{x}_i$ and its $j$-th neighbor $\b{x}_{i,j}$, we fit a line, denoted by $\ell_{i,j}$, embedded in $\mathbb{R}^d$. 
We move every neighbor $\b{x}_{i,j}$ to the point $\b{x}_i$ along the line $\ell_{i,j}$ until we get very close to $\b{x}_i$: 
\begin{align}\label{equation_x_neighbor_limit}
& \b{x}'_{i,j} := \lim_{\varepsilon \rightarrow 0^+} \Big( (1-\varepsilon)\, \b{x}_{i} + \varepsilon\, \b{x}_{i,j} \Big). 
\end{align}
According to Definitions \ref{definition_osculating_circle} and \ref{definition_osculating_hypersphere}, by tending all $d$ neighbors $\{\b{x}_{i,j}\}_{j=1}^d$ to $\b{x}_i$, the hypersphere $S_i$ converges to an osculating hypersphere for a hypothetical hypersurface on which the point $\b{x}_i$ exists. 
The osculating hypersphere $S_i$ is obtained by fitting a $(d-1)$-sphere to $\b{x}_i$ and $\{\b{x}'_{i,j}\}_{j=1}^d$.
If $x^r$ denotes the $r$-th dimension of $\b{x}$, the expression of $S_i$ is obtained by fitting a hypersphere, embedded in $\mathbb{R}^d$, to these points:
\begin{align}
\text{det}\left(
\begin{bmatrix}
\sum_{r=1}^d (x^r)^2 & x^1 & x^2 & \dots & x^d & 1\\
\sum_{r=1}^d (x_i^r)^2 & x_i^1 & x_i^2 & \dots & x_i^d & 1\\
\sum_{r=1}^d (x_{i,1}^r)^2 & x_{i,1}^1 & x_{i,1}^2 & \dots & x_{i,1}^d & 1\\
\vdots & \vdots & \vdots & & \vdots & \vdots \\
\sum_{r=1}^d (x_{i,d}^r)^2 & x_{i,d}^1 & x_{i,d}^2 & \dots & x_{i,d}^d & 1
\end{bmatrix}
\right) \overset{\text{set}}{=} 0,
\end{align}
where $\text{det}(.)$ denotes the determinant of matrix. 
This osculating hypersphere $S_i$ passes through $\b{x}_i$ and $\{\b{x}'_{i,j}\}_{j=1}^k$ (see Fig. \ref{figure_hypercap}). 
We do this procedure for all the $n$ points of dataset. Hence, we have have fitted $n$ number of osculating hyperspheres, $\{S_i\}_{i=1}^n$, on a hypothetical hypersurface. Every $S_i$ is a $(d-1)$-sphere embedded in $\mathbb{R}^d$. 

\begin{figure}[!h]
\centering
\includegraphics[width=5.5in]{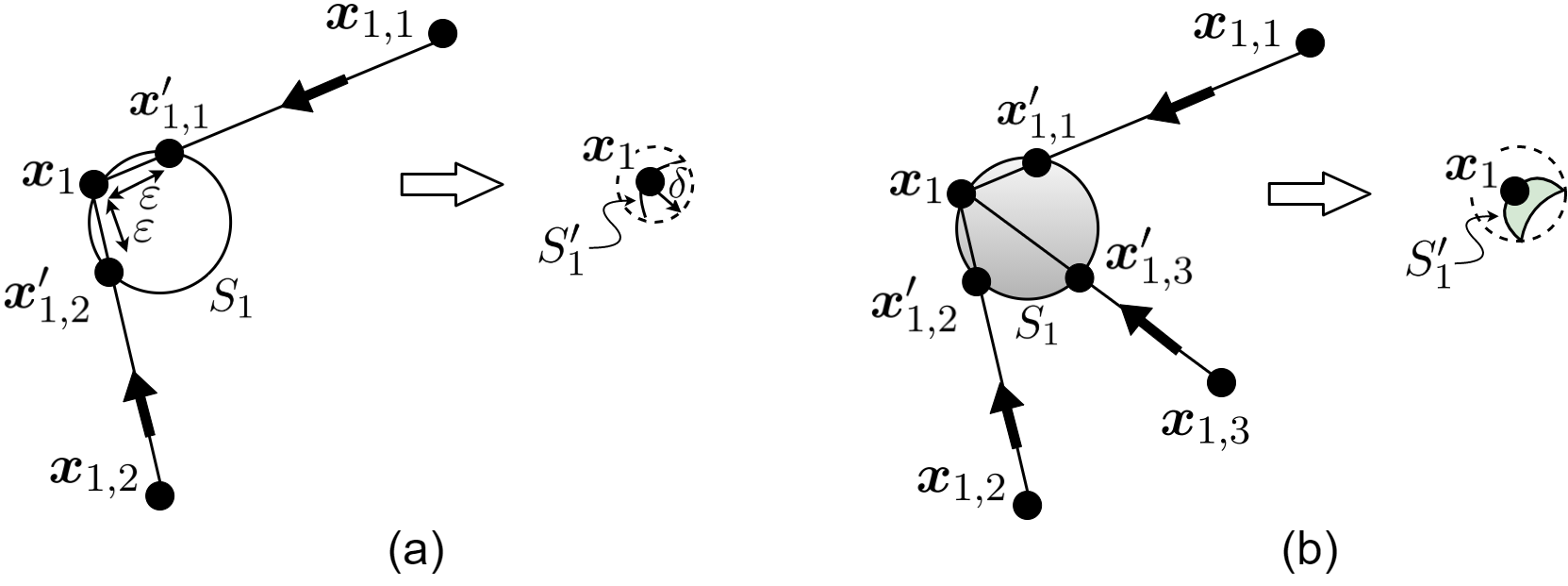}
\caption{fitting an osculating hypersphere on $\b{x}_1$ using its neighbors and then applying surgery to extract a hyper-cap in (a) $\mathbb{R}^2$ and (b) $\mathbb{R}^3$.}
\label{figure_hypercap}
\end{figure}

\subsection{Surgery and Hyper-caps}

\hfill\break
\hfill\break
Now, we borrow the idea of surgery theory \cite{browder2012surgery} to have hyper-caps on the $n$ data points. 
Consider an osculating hypersphere $S_i$. We know $\b{x}_i \in S_i$. We do surgery on $S_i$ to divide it to two manifolds $S'_i$ and $S''_i$ where:
\begin{align}
& S'_i \sqcup S''_i = S_i, \\
& S'_i \cap S'_i = \varnothing, \\
& \b{x}_i \in S'_i, \quad \b{x}_i \not\in S''_i, 
\end{align}
where $\sqcup$ denotes the disjoint union. We do this surgery in a way that $S'_i$ is much smaller than $S''_i$. In other words, $S'_i$ only contains $\b{x}_i$ and merely a small neighborhood of $\b{x}_i$ on $S_i$:
\begin{align}
& S'_i := \{\b{x}\, |\, \Vert\b{x} - \b{x}_i\Vert \leq \delta, \b{x} \in S_i\}, 
\end{align}
where $\delta$ is a very small positive number and $\delta \leq \varepsilon$ (cf. Eq. (\ref{equation_x_neighbor_limit})). This $S'_i$ is a hyper-cap containing $\b{x}_i$.
We do this for all $n$ osculating hyperspheres to have $n$ number of hyper-caps which are locally $(d-1)$-dimensional and are embedded in $\mathbb{R}^d$. 
The procedure of surgery and extraction of the hyper-cap is illustrated in Fig. \ref{figure_hypercap}.




\subsection{Hyper-cylinder Fitting}\label{section_hyper_cylinder_fitting}


\hfill\break
\hfill\break
Every two hyper-caps can be connected by a diagonal hyper-cylinder which is locally $(d-1)$-dimensional and embedded in $\mathbb{R}^d$. 
We do not use the entire hyper-cylinder but we do surgery on it \cite{browder2012surgery} to halve it from its main diagonal. This is shown in Fig. \ref{figure_hyper_cylinder} for $\mathbb{R}^3$. 
Let the hyper-cylinder connecting the hyper-caps $S'_i$ and $S'_j$ be denoted by $c_{i,j}$. 
Not all pairs of $S'_i$ and $S'_j$ should be connected by a partial hyper-cylinder. 
We choose the connected hyper-caps in a way that, except two hyper-caps (called the boundary hyper-caps), every hyper-cap is connected to exactly two other hyper-caps by a partial hyper-cylinder. 
For the two boundary hyper-caps, we can have two approaches:
\begin{enumerate}
\item In the first approach, each of the two boundary hyper-caps is connected to a nearby hyper-cap and the other boundary hyper-cap. 
We connect the two boundary hyper-caps using an orientable $(d-1)$-hyper-strip, with boundary, in a way that it does not intersect the connecting partial hyper-cylinders and the hyper-caps. As the dataset is bounded and falls in some ball in the Euclidean space (see Assumption \ref{assumption_dataset_in_ball}), there always exists such non-intersecting hypersurface. 
In this approach, every hyper-cap is connected to exactly two other hyper-caps. 
\item In the second approach, we assume that two extra data points exist at infinity but in opposite directions. For example, the two extra points can be $[0,0, \dots, 0, \infty]^\top \in \mathbb{R}^d$ and $[0,0, \dots, 0, -\infty]^\top \in \mathbb{R}^d$. 
Each of the two boundary hyper-caps is connected to a nearby hyper-cap and one of the extra points at infinity. We connect the boundary hyper-cap to the point at infinity with an orientable $(d-1)$-hyper-strip, with boundary. The two added hyper-strips for the two boundary hyper-caps are chosen in a way that they do not intersect, i.e., they are disjoint.
In this approach, every intermediate hyper-cap is connected to exactly two other hyper-caps but the boundary hyper-caps are each connected to only one other intermediate hyper-cap.
\end{enumerate}
An example for both approaches in $\mathbb{R}^2$ is shown in Fig. \ref{figure_connecting_hypercaps}.

\begin{figure}[!h]
\centering
\includegraphics[width=3.5in]{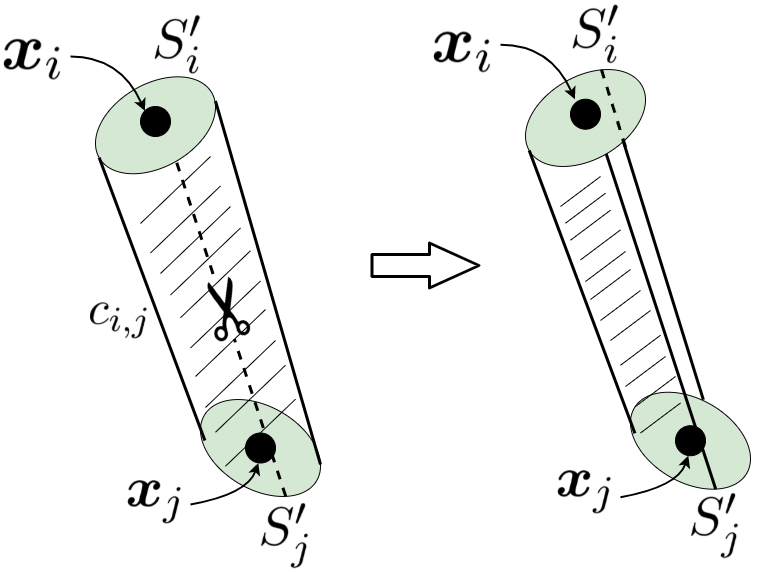}
\caption{Connecting hyper-caps $S'_i$ and $S'_j$ with a tilted/diagonal hyper-cylinder $c_{i,j}$ and then halving it diagonally with surgery, in $\mathbb{R}^3$. Note that the hyper-caps in this figure should be imagined to be curvy like a bowl.}
\label{figure_hyper_cylinder}
\end{figure}

\begin{figure}[!h]
\centering
\includegraphics[width=5.5in]{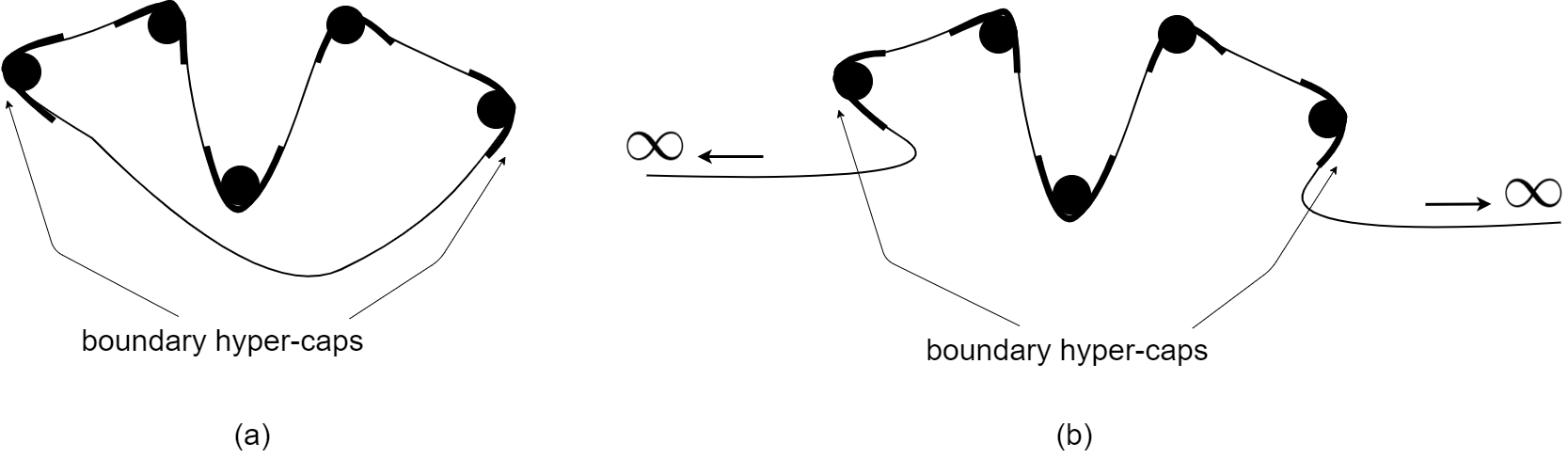}
\caption{Connecting the hyper-caps, in $\mathbb{R}^2$, in (a) the first approach and (b) the second approach.}
\label{figure_connecting_hypercaps}
\end{figure}

\subsection{The Embedded Hypersurface}

\hfill\break
\hfill\break
We denote the final desired hypersurface by $\mathcal{W}$.
Let $\mathcal{E}$ denote a set of pairs of indices for the hyper-caps which are connected by the connecting hyper-cylinders.
Let $i_a$ and $i_b$ be the indices of the boundary hyper-caps where $i_a, i_b \in \{1, \dots, n\},  i_a \neq i_b$. 
In the first approach, we denote the hyper-strip connecting the two boundary hyper-caps by $s_{i_a, i_b}$.
In this approach, the final hypersurface is obtained by the disjoint union of hyper-caps and the connecting hyper-cylinders:
\begin{align}\label{equation_hypersurface_approach_1}
& \mathcal{W} := \Big( \bigsqcup_{i=1}^n S'_i \Big) \sqcup \Big(\bigsqcup_{(i,j) \in \mathcal{E}} c_{i,j} \Big) \sqcup s_{i_a, i_b}.
\end{align}
In the second approach, we denote the two hyper-strips for connecting the two boundary hyper-caps to some infinity by $s_{i_a}$ and $s_{i_b}$.
In this approach, the hypersurface is obtained by the disjoint union of hyper-caps, the connecting hyper-cylinders, and the two hyper-strips.
\begin{align}\label{equation_hypersurface_approach_2}
& \mathcal{W} := \Big( \bigsqcup_{i=1}^n S'_i \Big) \sqcup \Big(\bigsqcup_{(i,j) \in \mathcal{E}} c_{i,j} \Big) \sqcup s_{i_a} \sqcup s_{i_b}. 
\end{align}
The obtained hypersurface $\mathcal{W}$ by either approach is $(d-1)$-dimensional locally, i.e., it is homeomorphic to $\mathbb{R}^{d-1}$. 

\section{A Numerical Example}\label{section_example}

Here, we provide a numerical example for the proposed method in $\mathbb{R}^2$. 
Consider a two-dimensional dataset with $n=3$ points $\b{x}_1 = [0,0]^\top$, $\b{x}_2 = [0,1]^\top$, and $\b{x}_3 = [1,0]^\top$. 
For the osculating hypersphere for $\b{x}_1$, we tend $\b{x}_{1,1} = \b{x}_2$ and $\b{x}_{1,2} = \b{x}_3$ to $\b{x}_1$. Hence, we have $\b{x}'_{1,1} = [0, \varepsilon]^\top$ and $\b{x}'_{1,2} = [\varepsilon, 0]^\top$ and the osculating hypersphere $S_1$ at $\b{x}_1$ is:
\begin{align}
&\text{det}\left(
\begin{bmatrix}
(x^1)^2+(x^2)^2 & x^1 & x^2 & 1\\
0 & 0 & 0 & 1\\
\varepsilon^2 & 0 & \varepsilon & 1\\
\varepsilon^2 & \varepsilon & 0 & 1\\
\end{bmatrix}
\right) 
\overset{(a)}{=} 
((x^1)^2+(x^2)^2) (-\varepsilon^2) - x^1(-\varepsilon^3) + x^2(\varepsilon^3) \nonumber\\
&~~~~~~~~= \varepsilon^2 \big(\!-(x^1)^2-(x^2)^2 + \varepsilon x^1 + \varepsilon x^2\big) \overset{\text{set}}{=} 0 
\,\,\implies\,\, (x^1)^2 + (x^2)^2 - \varepsilon x^1 - \varepsilon x^2 = 0, \label{equation_numerical_example_S_prime_1}
\end{align}
where $(a)$ is because of the Laplace expansion of determinant. 

Likewise, for $\b{x}_2$, we have $\b{x}_{2,1} = \b{x}_1$, $\b{x}_{2,2} = \b{x}_3$, $\b{x}'_{2,1} = [0, 1-\varepsilon]^\top$ and $\b{x}'_{2,2} = [\varepsilon, 1-\varepsilon]^\top$ and the osculating hypersphere $S_2$ is:
\begin{align}
&\text{det}\left(
\begin{bmatrix}
(x^1)^2+(x^2)^2 & x^1 & x^2 & 1\\
0 & 1 & 0 & 1\\
(1-\varepsilon)^2 & 0 & (1-\varepsilon) & 1\\
\varepsilon^2 + (1-\varepsilon)^2 & \varepsilon & (1-\varepsilon) & 1\\
\end{bmatrix}
\right) 
\nonumber \\ 
&= 
((x^1)^2+(x^2)^2) (-\varepsilon (1-\varepsilon)) - x^1(- \varepsilon^2 (1-\varepsilon)) + x^2(\varepsilon^2 + \varepsilon (1-\varepsilon)^2) - 1(\varepsilon^2 (1-\varepsilon)) \nonumber\\
&= \varepsilon (1-\varepsilon) \big(\! -(x^1)^2 - (x^2)^2 + \varepsilon x^1 + (1-\varepsilon) x^2 - \varepsilon \big) + \varepsilon^2 x^2 \overset{\text{set}}{=} 0 \nonumber\\
&~~~~~~~~~\implies\,\, -(x^1)^2 - (x^2)^2 + \varepsilon x^1 + (1-\varepsilon) x^2 - \varepsilon + \frac{\varepsilon}{(1-\varepsilon)} x^2 = 0 \nonumber\\
&~~~~~~~~~\implies\,\, (x^1)^2 + (x^2)^2 - \varepsilon x^1 - (1+\frac{\varepsilon^2}{1-\varepsilon}) x^2 - \varepsilon = 0. \label{equation_numerical_example_S_prime_2}
\end{align}
The osculating hypersphere $S_3$ can also be obtained similarly. 

We apply surgery on the $S_1$, $S_2$, and $S_3$ to obtain $S'_1$, $S'_2$, and $S'_3$, respectively. Here, we show how to connect the hyper-caps $S'_1$ and $S'_2$ as an example. Hence, we show their surgery for this connection only. 
For the connection of $S'_1$ and $S'_2$, the surgery of $S'_1$ with radius $\delta$ is done as:
\begin{align*}
&(x^1)^2 + (x^2)^2 - \varepsilon x^1 - \varepsilon x^2 \big|_{x^2 = \delta} \overset{(\ref{equation_numerical_example_S_prime_1})}{=} 0 \,\,\implies\,\, (x^1)^2 - \varepsilon x^1 + \delta^2 - \varepsilon \delta = 0 \\
&\implies\,\, x^1 = \frac{\varepsilon - \sqrt{\varepsilon^2 - 4 \delta^2 + 4 \varepsilon \delta}}{2}.
\end{align*}
For the connection of $S'_1$ and $S'_2$, the surgery of $S'_2$ with radius $\delta$ is done as:
\begin{align*}
&(x^1)^2 + (x^2)^2 - \varepsilon x^1 - (1+\frac{\varepsilon^2}{1-\varepsilon}) x^2 - \varepsilon \big|_{x^2 = 1-\delta} \overset{(\ref{equation_numerical_example_S_prime_2})}{=} 0 \\
&\,\,\implies\,\, (x^1)^2 + (1-\delta)^2 - \varepsilon x^1 - (1+\frac{\varepsilon^2}{1-\varepsilon}) (1-
\delta) - \varepsilon = 0 \\
&\,\,\implies\,\, (x^1)^2 - \varepsilon x^1 + (1-\delta) (-\frac{\varepsilon^2}{1-\varepsilon}-\delta) - \varepsilon = 0 \\
&\implies\,\, x^1 = \frac{\varepsilon - \sqrt{\varepsilon^2 - 4 (1-\delta) (-\frac{\varepsilon^2}{1-\varepsilon}-\delta) + 4 \varepsilon}}{2}.
\end{align*}
In the $\mathbb{R}^2$ space, the connecting hyper-cylinder, which is locally $(d-1)$-dimensional locally, is simplified to a connecting line. Hence, the connecting hyper-cylinder between $S'_1$ and $S'_2$ is the line connecting the following two obtained points:
\begin{align*}
&y_1 := 
\begin{bmatrix}
y_1^1 \\
y_1^2
\end{bmatrix}
=
\begin{bmatrix}
0.5\varepsilon - 0.5\sqrt{\varepsilon^2 - 4 \delta^2 + 4 \varepsilon \delta} \\
\delta
\end{bmatrix}
, \\
&y_2 := 
\begin{bmatrix}
y_2^1 \\
y_2^2
\end{bmatrix}
=
\begin{bmatrix}
0.5\varepsilon - 0.5\sqrt{\varepsilon^2 - 4 (1-\delta) (-\frac{\varepsilon^2}{1-\varepsilon}-\delta) + 4 \varepsilon} \\
1-\delta
\end{bmatrix}.
\end{align*}
Let this line be expressed as:
\begin{align*}
a_1 x^1 + a_2 = x^2,
\end{align*}
where $a_1$ and $a_2$ are the coefficients.
This line can be obtained by the following system of equations:
\begin{align*}
\begin{bmatrix}
y_1^1 & 1\\
y_2^1 & 1
\end{bmatrix}
\begin{bmatrix}
a_1\\
a_2
\end{bmatrix}
= 
\begin{bmatrix}
y_1^2\\
y_2^2
\end{bmatrix},
\end{align*}
which can be a least squares problem. However, as the coefficient matrix has a full rank, the coefficients are easily calculated as:
\begin{align*}
\begin{bmatrix}
a_1\\
a_2
\end{bmatrix}
= 
\begin{bmatrix}
0.5\varepsilon - 0.5\sqrt{\varepsilon^2 - 4 \delta^2 + 4 \varepsilon \delta} & 1\\
0.5\varepsilon - 0.5\sqrt{\varepsilon^2 - 4 (1-\delta) (-\frac{\varepsilon^2}{1-\varepsilon}-\delta) + 4 \varepsilon} & 1
\end{bmatrix}^{-1}
\begin{bmatrix}
\delta \\
1 - \delta
\end{bmatrix}.
\end{align*}
Hence, we have the connecting hyper-cylinder (or line here) between $S'_1$ and $S'_2$. Similarly, we can calculate the surgery and hyper-cylinder for other hyper-caps. Putting all hyper-caps and the hyper-cylinders (i.e., lines here) gives the embedded $1$-hypersurface submanifold. 

\section{Geometrical Characteristics of the Embedded Hypersurface}\label{section_geometrical_characteristics}

\begin{proposition}[On having boundary]
The obtained embedded hypersurface $\mathcal{W}$ is a manifold with boundary. 
\end{proposition}
\begin{proof}[Proof (sketch)]
The obtained hypersurface is composed of union of the hyper-caps and the connecting partial hyper-cylinders. The partial hyper-cylinders are obtained from a surgery on the tilted hyper-cylinders by halving it diagonally (see Fig. \ref{figure_hyper_cylinder}). Therefore, the connecting partial hyper-cylinders have boundary. Hence, the hypersurface also has boundary.
\end{proof}

\begin{proposition}[On the topology of hypersurface]\label{proposition_hypersurface_homeomorphic_to_strip}
In the first approach of connecting hyper-caps (see Section \ref{section_hyper_cylinder_fitting}), the obtained hypersurface $\mathcal{W}$ is homeomorphic to an orientable hyper-strip loop.
In the second approach, $\mathcal{W}$ is homeomorphic to an unbounded orientable hyper-strip.
\end{proposition}
\begin{proof}[Proof (sketch)]
In the first approach, every intermediate hyper-cap is connected to exactly two other hyper-caps. Every boundary hyper-cap is connected to a neighbor hyper-cap and the other boundary hyper-cap (see Fig. \ref{figure_connecting_hypercaps}). Hence, it forms a loop structure in topology which is homeomorphic to a hyper-strip loop. Proof for why it is orientable will be provided in the proof of Corollary \ref{corollary_W_orientable}.
In the second approach, the intermediate hyper-caps are connected by hyper-cylinders. The two boundary hyper-caps are connected to some infinity by hyper-strips. Therefore, the resulted hypersurface has boundary (because of the partial hyper-cylinders and the hyper-strips) and is homeomorphic to an orientable hyper-strip. It is unbounded because it goes to infinity. Orientability will be shown in the proof of Corollary \ref{corollary_W_orientable}.
\end{proof}

\begin{corollary}[On smoothness]
In both approaches of connecting hyper-caps (see Section \ref{section_hyper_cylinder_fitting}), the obtained hypersurface $\mathcal{W}$ is smooth and differentiable.  
\end{corollary}
\begin{proof}[Proof (sketch)]
All the hyper-caps, the partial hyper-cylinders, and the hyper-strips, used in the two approaches, are smooth. According to Eqs. (\ref{equation_hypersurface_approach_1}) and (\ref{equation_hypersurface_approach_2}), the embedded hypersurface is composed of these smooth elements; hence, it is smooth.
\end{proof}

\begin{corollary}[On boundedness]
Assuming that the points of dataset lie in some finite ball in $\mathbb{R}^d$ (see Assumption \ref{assumption_dataset_in_ball}), the obtained hypersurface $\mathcal{W}$ is bounded in the first approach of connecting hyper-caps (see Section \ref{section_hyper_cylinder_fitting}).
In the second approach, the obtained hypersurface $\mathcal{W}$ is unbounded, regardless of whether dataset is bounded or not. 
\end{corollary}
\begin{proof}[Proof (sketch)]
In the first approach, every hyper-cap is connected to two other hyper-caps (see Fig. \ref{figure_connecting_hypercaps}). Assuming that the points of dataset lie in some finite ball in $\mathbb{R}^d$, the topology of the resulted hypersurface is bounded. In the second approach, the hypersurface goes to some infinity by the two hyper-strips. Hence, it is unbounded. 
\end{proof}

\begin{corollary}[On orientability]\label{corollary_W_orientable}
In both approaches of connecting hyper-caps (see Section \ref{section_hyper_cylinder_fitting}), the obtained hypersurface $\mathcal{W}$ is orientable.  
\end{corollary}
\begin{proof}[Proof (sketch)]
In the first approach, every hyper-cap is connected to exactly two other hyper-caps (see Fig. \ref{figure_connecting_hypercaps}). Hence, it forms a loop structure in topology. After the surgery of tilted hyper-cylinders (see Fig. \ref{figure_hyper_cylinder}), we do not tilt the partial hyper-cylinders. Hence, the connectors of hyper-caps are not tilted and therefore, they are all orientable (in contrast to the M{\"o}bius strip \cite{mobius1863theorie,mobius1865ueber} which is tilted and non-orientable). The hyper-caps are also orientable. 
We choose the hyper-strip connecting the boundary hyper-caps to be orientable.
Therefore, the resulted $\mathcal{W}$ is orientable. 
Likewise, in the second approach, the hyper-caps and the connecting hyper-cylinders are orientable. We choose the two hyper-strips to be orientable, i.e., we do not tilt them. Hence, the resulted $\mathcal{W}$ is orientable also in the second approach. 
\end{proof}

\begin{corollary}[On compactness]
In both approaches of connecting hyper-caps (see Section \ref{section_hyper_cylinder_fitting}), the obtained hypersurface $\mathcal{W}$ is not compact.   
\end{corollary}
\begin{proof}[Proof (sketch)]
According to Proposition \ref{proposition_hypersurface_homeomorphic_to_strip}, $\mathcal{W}$ has boundary in both approaches because both strip loop and hyperplane with boundary have boundary. 
Hence, according to Definition \ref{definition_compact_topological_space}, $\mathcal{W}$ is not compact. 
\end{proof}

\begin{proposition}[On injectivity and invertibility]
In both approaches of connecting hyper-caps (see Section \ref{section_hyper_cylinder_fitting}), the embedding of the obtained hypersurface $\mathcal{W}$ is injective. 
\end{proposition}
\begin{proof}[Proof (sketch)]
In the first approach, the strip loop does not intersect itself because we choose the set $\mathcal{E}$ in a way that every hyper-cap is connected to exactly two other hyper-caps and as we have freedom in choosing $\mathcal{E}$, we can choose it in a way so the resulted $\mathcal{W}$ does not intersect itself. Similar discussion holds for the second approach where the obtained hypersurface $\mathcal{W}$ does not intersect itself. Hence, the embedding of the $(d-1)$-hypersurface in $\mathbb{R}^d$ is invertible and injective.
\end{proof}

\section{Induction for Extension to Lower Submanifold Dimensionalities}\label{section_induction}

So far, we showed that manifold hypothesis holds for the embedding dimensionality of at least $d-1$ where $d$ is the dimensionality of data. Now, we extend the embedding dimensionality to lower dimensionalities using induction.
Consider a dataset in the $d$-dimensional Euclidean space, i.e., $\mathbb{R}^d$. 
The induction goes on as follows:
\begin{itemize}
\item Step 1: Using the approach explained in Section \ref{section_hypersurface_embedding}, we can fit an embedding hypersurface with local dimensionality $d-1$; hence, the dataset can be embedded on a $(d-1)$-dimensional hypersurface. 
\item Step 2: We represent the $d$-dimensional data points of the dataset in the obtained space of the $(d-1)$-dimensional embedded space. For this representation, we perform as follows. Note that the $(d-1)$-dimensional fitted hypersurface is homeomorphic to the $(d-1)$-dimensional Euclidean space, $\mathbb{R}^{d-1}$. Then, we represent the points of dataset in the $\mathbb{R}^{d-1}$ which is homeomorphic to the fitted hypersurface. 
\item Step 3: Now, we have a represented dataset with dimensionality $d-1$. In the obtained $\mathbb{R}^{d-1}$ for representation of data, we go to step 1 but with dimensionality $d-1$. Hence, we can fit an embedding hypersurface with local dimensionality $d-2$ so the dataset can be embedded on a $(d-2)$-dimensional hypersurface. 
\item Step 4: Similar to step 2, we represent the $(d-1)$-dimensional data points of the dataset in the obtained space of the $(d-2)$-dimensional embedded space. 
\item We continue this procedure by embedding a $(d-j)$-dimensional dataset on a $(d-j-1)$-dimensional hypersurface, for all $j \in \{0, 1, \dots, d-2\}$. 
\item The base of induction is embedding data on a $1$-hypersurface whose local dimensionality is one. In this case, the hypersurface is a curve passing through the data points in the Euclidean space. 
\end{itemize}
An illustration of this induction can be found in Fig. \ref{figure_induction}. As this figure shows, the induction forms a hierarchy or a pyramid of submanifold embeddings. 

\begin{figure}[!h]
\centering
\includegraphics[width=5.5in]{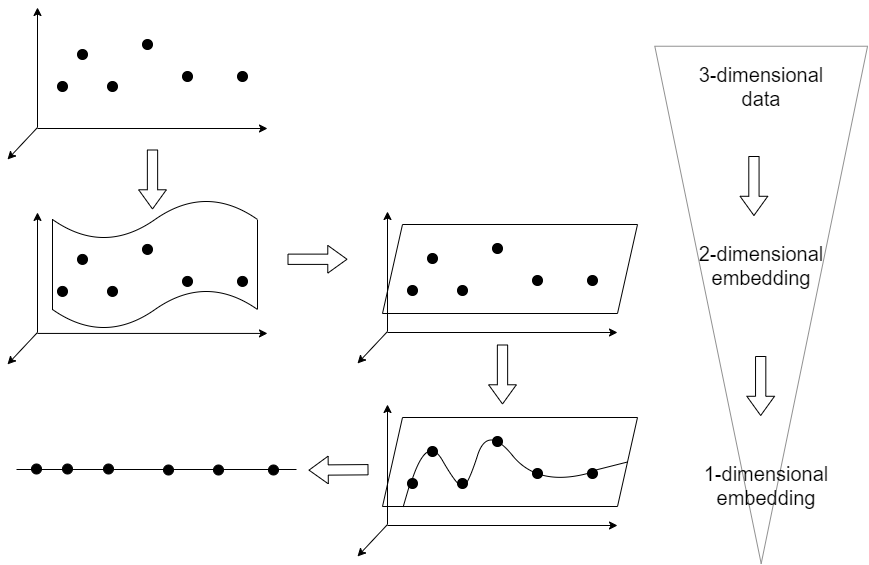}
\caption{Illustration of the induction in a pyramid structure for embedding three dimensional data points. First, a $2$-hypersurface is fitted to data. Then, we represent dataset as two dimensional points in the homeomorphic $\mathbb{R}^2$. Then, a $1$-hypersurface, i.e. a curve, is fitted to data in the $\mathbb{R}^2$. Note that the fitted $1$-hypersurface can also be homeomorphic to $\mathbb{R}^1$ for one dimensional representation of data.}
\label{figure_induction}
\end{figure}

\begin{corollary}
As the above induction states, we can embed a $d$-dimensional dataset on multiple hypersurfaces with local dimensionalities $\{1, 2, \dots, d-1\}$. Therefore, the manifold hypothesis holds for all submanifold dimensionalities less than $d$. 
\end{corollary}

\begin{remark}
One or several of the local dimensionalities $\{1, 2, \dots, d-1\}$ are the best embedding dimensionalities for the best representation of data or discrimination of classes. Although manifold hypothesis holds for all the dimensionalities $\{1, 2, \dots, d-1\}$, finding the best embedding dimensionality among these is out of scope of the manifold hypothesis and requires an algorithm. Various algorithms and methods, such as the scree plot \cite{cattell1966scree}, exist for finding the best embedding dimensionality. 
\end{remark}

\section{Discussions on the Linearity and Structure of Data}\label{section_discussions}

\subsection{Discussion on the Linearity of Data}

\hfill\break
\hfill\break
The procedure of fitting a hypersurface to data can also be interpreted as a kind of nonlinear regression. 
Note that if the data points actually lie on a linear hyperplane, we will have flat hyper-caps and the fitted hyperplane is the embedded hypersurface itself. In this case of linearity, only a few of the osculating hyperspheres are sufficient for fitting the hyperplane and the other osculating hyperspheres will be redundant. Hence, the linear case of our method is relevant to few-shot learning \cite{wang2020generalizing} in which a few data instances are used for learning. 

\subsection{Discussion on the Structure of Data}

\hfill\break
\hfill\break
Another discussion we can have is that the less curvature the fitted hypersurfaces in the induction have, the larger the osculating hyperspheres will be. This is because the curvature of the hypersurface is reciprocal to the radius of the osculating hypersphere (see Definition \ref{definition_osculating_hypersphere}). 
If data points do not have a specific structure, such as white noise, the hypersurface needs to pass through all the points which do not have a specific structure. In this case, the curvature of the fitted hypersurface is very large at different parts of the hypersurface. Hence, if most of the osculating hyperspheres have a small radius, the dataset seems more like the white noise, without any specific structure. 


\section{Conclusion}\label{section_conclusion}

This was a paper concentrating on the manifold hypothesis which states that data points lie on an embedded submanifold. We showed that the manifold hypothesis holds at least for the embedding dimensionality $d-1$. 
We found this embedded hypersurface by fitting an osculating hypersphere to every points, using its neighbors, then applying surgery on hyperspheres to obtain hyper-caps, and connecting the hyper-caps using partial hyper-cylinders. We also discussed the geometrical characteristics of the embedded hypersurface. 
We extended the embedding dimensionality to lower embedding dimensionalities to show the validity of manifold hypothesis for embedding dimensionalities $\{1, 2, \dots, d-1\}$. Some discussion were also provided for the linearity and structure of data. 
A possible future work is to discuss manifold hypothesis and the proposed method for generalization to out-of-sample data.  

\section*{Acknowledgement}

We hugely thank Prof. Spiro Karigiannis at the Department of Pure Mathematics in the University of Waterloo for the fruitful discussions about this paper. 

\bibliographystyle{amsalpha}
\bibliography{refs}



\hfill\break
Benyamin Ghojogh\\
Department of Electrical and Computer Engineering, University of Waterloo, Canada\\
\texttt{bghojogh@uwaterloo.ca}

Fakhri Karray\\
Department of Electrical and Computer Engineering, University of Waterloo, Canada\\
\texttt{karray@uwaterloo.ca}

Mark Crowley\\
Department of Electrical and Computer Engineering, University of Waterloo, Canada\\
\texttt{mcrowley@uwaterloo.ca}

\end{document}